\documentclass[journal,twoside,web]{ieeecolor}
\usepackage{lcsys}
\usepackage[utf8]{inputenc}

\usepackage{blindtext}

\usepackage{booktabs}

\usepackage{amsmath} 
\usepackage{bbm}
\usepackage{amssymb}
\usepackage{mathtools}

\usepackage{amsthm}

\usepackage{comment}

\theoremstyle{definition}
\newtheorem{definition}{Definition}
\newtheorem{theorem}{Theorem}
\newtheorem{lemma}{Lemma}
\newtheorem{assumption}{Assumption}

\newtheorem{problem}{Problem}

\usepackage{hyperref}
\hypersetup{
    colorlinks=true,
    linkcolor=black,
    filecolor=black,      
    urlcolor=blue,
    citecolor=blue,
    }
\usepackage{cite}

\newcommand{\blue}[1]{\textcolor{black}{#1}}

\usepackage{todonotes}

\title{
Constraint-Driven Optimal Control of Multi-Agent Systems: A Highway Platooning Case Study
}

\author{Logan E. Beaver, \emph{Student Member, IEEE},  Andreas A. Malikopoulos, \emph{Senior Member, IEEE} 
	\thanks{This research was supported by ARPAE's NEXTCAR program under the award number DE-AR0000796. This support is gratefully acknowledged.}
	\thanks{L.E. Beaver and A.A. Malikopoulos are with the Department of Mechanical Engineering, University of Delaware, Newark, DE, USA (emails: lebeaver@udel.edu, andreas@udel.edu).}%
}

\begin{document}

\maketitle
\thispagestyle{empty}

\begin{abstract}

Platooning has been exploited as a method for vehicles to minimize energy consumption.
In this article, we present a constraint-driven optimal control framework that yields emergent platooning behavior for connected and automated vehicles operating in an open transportation system.
Our approach combines recent insights in constraint-driven optimal control with the physical aerodynamic interactions between vehicles in a highway setting.
The result is a set of equations that describes when platooning is an appropriate strategy, as well as a descriptive optimal control law that yields emergent platooning behavior.
Finally, we demonstrate these properties in simulation. 
\end{abstract}


\section{Introduction}

\IEEEPARstart{M}{ulti-agent} systems have attracted considerable attention in many applications due to their natural parallelization, general adaptability, and ability to self-organize \cite{Oh2017}.
This has proven useful in many applications of complex systems \cite{Malikopoulos2016}, such as emerging mobility systems \cite{Malikopoulos2020}, construction \cite{Lindsey2012ConstructionTeams}, and surveillance \cite{Corts2009}.
A recent push in constraint-driven control has brought the idea of long-duration autonomy to the forefront of multi-agent systems research \cite{Egerstedt2018RobotAutonomy}.
For long-duration autonomy tasks, robots are left to interact with their environment on timescales significantly longer than what can be achieved in a laboratory setting.
These approaches necessarily emphasize safe energy-minimizing control policies for the agents, whose behaviors are driven by interactions with the environment. Several applications of constraint-driven multi-agent control have been explored recently \cite{Notomista2019AnSystems,Ibuki2020Optimization-BasedBodies,Beaver2020Energy-OptimalConstraints}.

In this article, we propose a constraint-driven approach to generate emergent platooning behavior in a fleet of connected and automated vehicles (CAVs) operating in highway conditions.
Platooning behavior has been of particular interest due to the high potential for energy savings over long distances.
Early results by Athans \cite{athans1968} laid the groundwork for more recent results for highway driving \cite{shladover1991}, cooperative adaptive cruise control \cite{wang2017}, and mixed-traffic platooning \cite{mahbub2021_platoonMixed}.
We believe that platoon formation for long-distance highway travel is a natural fit for constraint-driven control.
There are several approaches to optimal platoon formation in the literature.
In one example, the authors sought to optimally position differential drive robots in an echelon formation such that energy lost to drag was minimized \cite{Bedruz2019DynamicOptimization}.
Reynolds' flocking rules were applied to highway vehicles in \cite{Fredette2017Fuel-SavingControl}, which sought to minimize energy consumption while maintaining a desired speed.
Energy-efficient flocking was also proposed for a system of flying robots in $\mathbb{R}^2$ \cite{Yang2016LoveControl}.
Previous approaches either construct a large multi-objective optimization problem to determine the next control action, or they apply sub-optimal consensus algorithms to reach a drafting configuration.
A recent review of these techniques is presented in \cite{Beaver2020AnFlockingb}.

Our approach, in contrast to existing work, is constraint-driven.
In our framework, agents seek to expend as little energy as possible subject to a set of task and safety constraints.
This set-theoretic approach to control is interpretable, i.e., the cause of an agent's action can be deduced by examining which constraints become active during operation.
By examining the conditions that lead to an empty feasible space, our framework also addresses when a vehicle should break away to form a new platoon or overtake the preceding vehicle. 
Our approach is totally decentralized, and thus it is well-suited to ``open systems," where agents may suddenly enter or leave.
We allow vehicles to arbitrarily enter or exit the system as long as their initial state is feasible and no other vehicles' safety constraint is violated.
This also allows vehicles to keep their final destination and arrival time private, which has the secondary benefit of guaranteeing privacy for all vehicles and their passengers.

The remainder of the article is organized as follows.
In Section \ref{sec:problem}, we formulate the platoon formation problem, and in Section \ref{sec:solution}, we present our decentralized constraint-driven control algorithm.
In Section \ref{sec:sim}, we validate our results by simulating $60$ vehicles, 
where vehicles randomly enter and leave the road network while the total number of vehicles is not known a priori.
Finally, we draw conclusions and propose some directions for future research in Section \ref{sec:conclusion}.

\section{Problem Formulation}\label{sec:problem}

We consider a set of CAVs traveling in a single-lane roadway. 
In particular, we consider an open transportation system that contains $N(t)\in\mathbb{N}$ CAVs indexed by the set $\mathcal{N}(t) = \{0, 1, 2, \dots, N(t)-1\}$, where $t\in\mathbb{R}$ is time and vehicle $i\in\mathcal{N}(t)\setminus\{0\}$ is in the aerodynamic wake of vehicle $i-1$.
We denote the state of each CAV $i\in\mathcal{N}(t)$ by $\mathbf{x}_i(t) = \begin{bmatrix} p_i(t), v_i(t) \end{bmatrix}^T,$ where $p_i(t), v_i(t) \in\mathbb{R}$ are the longitudinal position and speed of vehicle $i$ on its current path respectively.
Each vehicle obeys the second-order dynamics
\begin{align}
    \dot{p}_i(t) &= v_i(t), \notag\\
    \dot{v}_i(t) &= a_i(t) = u_i(t) - F_i\big(v_i(t),\hat{p}_i(t)\big), \label{eq:dynamics}
\end{align}
where $a_i(t)$ is acceleration, $u_i(t)$ is forward force imparted through the tires, $F_i\big(v_i(t), \hat{p}_i(t)\big)$ is the aerodynamic drag force acting on the vehicle, and $\hat{p}_i(t)$ is the relative position of CAV $i$, which we formally define later.
The objective of each vehicle is to minimize the effect of the external drag force, i.e.,
\begin{equation} \label{eq:objective}
    J_i(v_i(t), \hat{p}_i(t)) = \frac{1}{2} F_i\big(v_i(t), \hat{p}_i(t)\big)^2.
\end{equation}
By minimizing the external drag force of each vehicle, we have direct benefits in energy consumption.
Each vehicle $i$ is subject to state and control constraints, i.e.,
\begin{align}
    0 < v_{\min} &\leq v_i(t) \leq v_{\max}, \label{eq:vConst}\\
    a_{\min} &\leq a_i(t) \leq a_{\max},  \label{eq:aConst}
\end{align}
where \eqref{eq:vConst} is the lower and upper speed limit and \eqref{eq:aConst} is the maximum safe deceleration and acceleration.

We index the vehicles in descending order, i.e., $p_i(t) < p_j(t)$ for all $i > j, \,\, i,j\in\mathcal{N}(t)$.
Note, when a vehicle enters or exits the system, the CAVs can communicate to re-sequence themselves.
To simplify our notation, we introduce the \emph{relative state} coordinates.
\begin{definition}
For each vehicle $i\in\mathcal{N}(t)$, the relative states and control action are,
\begin{align}
    \hat{p}_i(t) &= 
    \begin{cases}
    p_i(t) & \text{ if } i = 0, \\
    p_i(t) - p_{i-1}(t) & \text{ if } i > 0,
    \end{cases} \label{eq:phat}\\
    \hat{v}_i(t) &= 
    \begin{cases}
    v_i(t) & \text{ if } i = 0, \\
    v_i(t) - v_{i-1}(t) & \text{ if } i > 0,
    \end{cases} \label{eq:vhat}\\
    \hat{a}_i(t) &= 
    \begin{cases}
    \dot{v}_i(t) & \text { if } i = 0, \\
    \dot{v}_i(t) - \dot{v}_{i-1}(t) & \text{ if } i > 0. \label{eq:ahat}
    \end{cases}
\end{align}
Note that in this coordinate system $\hat{p}_i(t) < 0$ for $i>0$.
While our approach does not impose a reference frame, it may be practical for a physical vehicle to measure \eqref{eq:phat} - \eqref{eq:vhat} directly, i.e., by using a proximity sensor. In that case, it may be advantageous for each vehicle to consider its current state as the center of a moving reference frame.

\end{definition}
To guarantee safety we impose the following safety constraint,
\begin{equation} \label{eq:safeDist}
    \hat{p}_i(t) + \delta \leq 0, \quad i\in\mathcal{N} \setminus \{0\},
\end{equation}
where $\delta \in\mathbb{R}_{>0}$ is the minimum safe bumper-to-bumper following distance.
However, na\"ively satisfying \eqref{eq:safeDist}  may still lead to unsafe scenarios and collisions.
Consider the case when $\hat{v}_i(t)$ is very large and vehicle $i-1$ applies $a_{i-1}(t) = a_{\min}$. 
To guarantee vehicle $i$ never ends up in an unsafe scenario, we impose an augmented safety constraint that guarantees sufficient stopping distance,
\begin{align} \label{eq:rearSafety}
    g_i^s(v_i, \hat{p}_i, \hat{v}_i) = 
    \begin{cases}
        \hat{p}_i(t) + \delta \quad\quad\quad\,\,\text{ if } \hat{v}_i(t) \leq 0,
        \\
        \hat{p}_i(t) + \delta + \hat{v}_i(t)\cdot\Big( \frac{v_{\min} - v_i(t)}{a_{\min}} \Big) \\
        \quad\quad\,+ \frac{\hat{v}_i(t)^2}{2 a_{\min}} \quad\quad\text{ if } \hat{v}_i(t) > 0,
    \end{cases}
\end{align}
for $i \in \mathcal{N} \setminus\{0\}$.
The case when $\hat{v}_i(t) > 0$ in \eqref{eq:rearSafety} is derived for CAV $i$ by assuming $i-1$ applies the maximum braking force until $v_{i-1}(t) = v_{\min}$ at some time $t_1$ and cruises with $a_{i-1}(t) = 0$, for $t \geq t_1$.
Then, \eqref{eq:rearSafety} allows CAV $i$ sufficient stopping distance to brake at $a_{\min}$ and maintain $\hat{p}_i(t) + \delta = 0$.
Note that the quadratic $\hat{v}_i(t)$ term is zero when $\hat{v}_i(t) = 0$ and increases up to a maximum at $\hat{v}_i(t) = v_{\max} - v_{\min}$.
Thus, satisfaction of \eqref{eq:rearSafety} always implies \eqref{eq:safeDist}.

Finally, each vehicle $i\in\mathcal{N}(t)$ has a terminal time $t_i^f$, which corresponds to the time that vehicle $i$ will exit the system, e.g., take an exit off the highway.
The value of $t_i^f$ is known only to vehicle $i$ and is not shared with any other vehicle.
This also ensures the privacy of vehicle $i$'s destination.
To ensure vehicle $i$ reaches its destination by time $t_i^f$, we impose an arrival deadline constraint,
\begin{equation} \label{eq:deadline}
    \big(S_i - p_i(t)\big) - \big(t_i^f - t\big) \, v_i(t) \leq 0,
\end{equation}
where $S_i$ is the position that $i$ will exit the system, e.g., via an off-ramp.
The arrival deadline constraint \eqref{eq:deadline} ensures that vehicle $i$ can reach its final destination by cruising at a constant speed.

Our objective in this article is the formation of platoons for long-duration autonomy, e.g., long-distance highway conditions.
Therefore, once vehicle $i$ joins a platoon, i.e., $\hat{v}_i = 0$ and $\hat{p}_i + \delta = 0$, other techniques, such as control barrier functions \cite{Zhu2019Barrier-function-basedConstraint} and consensus approaches \cite{Hendrickx2020TrajectoryPlatoons}, can be used to maintain the platoon.
To minimize \eqref{eq:objective} for our long-duration autonomy task, we impose the following assumptions.

\begin{assumption} \label{smp:aerodynamics}
    We neglect the effects of wind, and assume the air has constant properties.
    For vehicle $i\in\mathcal{N}\setminus\{0\}$ the drag force is zero at $F_i\big(0, \hat{p}_i(t)\big)$, increasing in $v_i(t)$, and decreasing in $\hat{p}_i(t)$.
    For vehicle $i=0$ the drag force is zero at $F_i\big(0, \hat{p}_i(t)\big)$, increasing in $v_i(t)$, and constant in $\hat{p}_i(t)$.
\end{assumption}

Assumption \ref{smp:aerodynamics} is the crux of our analysis, as it determines the signs of the derivatives of the cost function.
This assumption is not restrictive, and it can be relaxed if the partial derivatives of $F_i$ can be calculated or measured.
Different forms of $F_i$ will result in different vehicle behavior, and this can be interpreted as a data-driven forcing function.
For physical systems containing wind, eddies, and other turbulent effects, $F_i$ may be thought of as an average or filtered drag force; sensing the average aerodynamic forces between vehicles in real time is an active area of ongoing research \cite{Tagliabue2020TouchMultirotor}.

\begin{assumption} \label{smp:shortDrag}
    The drag acting on vehicle $i$ is only a function of the states of vehicles $i$ and $i-1$ for $i\in\mathcal{N}\setminus\{0\}$, and there are no external noise or disturbances.
\end{assumption}

\begin{assumption} \label{smp:comms}
    Communication between CAVs occurs instantaneously and noiselessly.
\end{assumption}

Assumptions \ref{smp:shortDrag} and \ref{smp:comms} idealize the environment in which the vehicles are operating to simplify the analysis.
Assumption \ref{smp:shortDrag} may be relaxed by expanding the drag model to include a time-varying component and additional interaction forces.
Likewise, Assumption \ref{smp:comms} can be relaxed by including delays, noise, and packet loss in a communication model.
If the disturbances and delays are bounded, then Assumption \ref{smp:comms} can be relaxed by shrinking the set of feasible actions using standard techniques, e.g., control barrier functions and differential inclusions \cite{Emam2019RobustTestbed}. 
However, we believe this adds significant analytical complexity without changing the fundamental results of our analysis.

\begin{assumption} \label{smp:tracking}
    Each vehicle $i\in\mathcal{N}(t)$ is equipped with a low-level controller that can track the desired acceleration, $a_i(t)$, by controlling the forward force applied to the CAV through $u_i(t)$.
\end{assumption}

Assumption \ref{smp:tracking} allows us to derive the kinematic motion of each CAV without directly considering the applied drag force.
This enables us to generate an analytic closed-form optimal trajectory for each vehicle without the numerical challenges associated with boundary-layer fluid dynamics.
This assumption can be relaxed by considering robust tracking, e.g., control barrier functions, or online learning to estimate and compensate for the aerodynamic interactions.

\section{Optimal Control with Gradient Flow} \label{sec:solution}

We employ \emph{gradient flow} to generate the control input for each vehicle.
This a gradient-based optimization technique, wherein each vehicle's control action is a gradient descent step.
This technique has been used successfully to control multi-agent constraint-driven systems \cite{Notomista2019Constraint-DrivenSystems,Wang2017SafetySystems}.
Our motivation for gradient flow is twofold:
first, planning a trajectory through a fluid boundary layer in real time requires significantly more computational power than what is available to a CAV.
Second, the exit time of the preceding vehicle is an unknown quantity, and so each vehicle cannot quantify the trade-off between accelerating to draft the preceding vehicle versus the energy savings of drafting.
Thus, we take a conservative approach where no vehicle will increase its energy consumption while traveling on the highway.
This approach yields conditions for \emph{when} platooning is an appropriate strategy in addition to \emph{how} the platoon should be formed.

As a first step, we define the set of \emph{safe control inputs} and show that it satisfies recursive feasibility \cite{Lofberg2012OopsMPC}.
For the remainder of the analysis, we omit the explicit dependence of state variables on $t$ when no ambiguity arises.
\begin{definition} \label{def:safeset}
For each vehicle $i\in\mathcal{N} \setminus \{ 0 \}$, the set of safe control inputs is
\begin{align}
    \mathcal{A}_i^s(v_i, \hat{p}_i, \hat{v}_i) = \Big\{ a \in \mathbb{R} ~:~ 
    &a_{\min} \leq a \leq a_{\max}, \notag\\
    &v_i = v_{\max} \implies a \leq 0 \notag\\
    &v_i = v_{\min} \implies a \geq 0 \notag\\
    &g_i^s = 0 \implies \frac{d}{dt}g_i^s \leq 0 \Big\}, \label{eq:safeControl}
\end{align}
where $g_i^s$ is the rear-end safety constraint \eqref{eq:rearSafety}, and $\frac{d}{dt} g_i^s \leq 0$ can be achieved through the control action, $a_i(t)$.
The safe set ensures the state, control, and safety constraints of vehicle $i$ are always satisfied.
\end{definition}

\begin{theorem}{(Recursive Feasibility)} \label{thm:recursive}
For any vehicle $i\in\mathcal{N}\setminus\{0\}$, if the variables $\hat{p}_i(t), \hat{v}_i(t), v_i(t)$ satisfy \eqref{eq:vConst} and \eqref{eq:rearSafety} at time $t_1\in\mathbb{R}$, then the set $\mathcal{A}_i^s$ is non-empty for all $t \geq t_1$.
\end{theorem}

\begin{proof}
To prove Theorem \ref{thm:recursive}, we show that a feasible control input always exists in the worst case scenario for vehicle $i$.
Let $\hat{v}_i(t_0) \geq 0$ and $a_{i-1}(t) = a_{\min}$ for $t\in[t_0, t_1)$ such that $v_{i-1}(t_1) = v_{\min}$ and $a_{i-1}(t) = 0$ for $t \geq t_1$.
We take the time derivative of \eqref{eq:rearSafety}, which yields
\begin{align}
    \hat{v}_i(t) 
    + \hat{v}_i(t) \cdot \Big( -\frac{a_i(t)}{a_{\min}} \Big)
    &+ \hat{a}_i(t)\Big( \frac{v_{\min} - v_i(t)}{a_{\min}} \Big) \notag\\
    &+ \Big(\frac{\hat{v}_i(t)\,\hat{a}_i(t)}{a_{\min}}\Big).
    \label{eq:safetyDerivative}
\end{align}
Over the interval $[t_0, t_1)$, $v_i(t) > v_{\min}$, and thus $a_i(t) = a_{\min}$ is a feasible control action.
This implies that $\hat{a}_i(t) = 0$, and evaluating \eqref{eq:safetyDerivative} implies
\begin{equation}
    \hat{v}_i(t) + \hat{v}_i(t) \Big(-1\Big) = 0,
\end{equation}
i.e., \eqref{eq:safetyDerivative} is identically zero, which implies that \eqref{eq:rearSafety} is constant.
Next, consider the interval $[t_1, t_2)$ such that $a_i(t) = a_{\min}$ for $t\in[t_1, t_2)$ and $v_i(t) = v_{\min}$ for $t \geq t_2$.
Thus, $a_i(t) = a_{\min}$ is a feasible control action, and evaluating \eqref{eq:safetyDerivative} implies
\begin{align}
    \hat{v}_i(t) 
    -\hat{v}_i(t) 
    + v_{\min} - v_i(t)
    + \hat{v}_i(t) &=
    2\hat{v}_i(t) - 2 \hat{v}_i(t) \notag\\
    &= 0,
\end{align}
which is identically zero over the entire interval.
This implies that \eqref{eq:rearSafety} is constant.
Finally, for $t>t_2$, $a_i(t) = 0$ is a feasible control action, which implies that $\hat{a}_i(t) = 0$, $\hat{v}_i(t) = 0$, and \eqref{eq:rearSafety} is constant.
Therefore, \eqref{eq:rearSafety} is constant for all $t > t_1$ in the worst-case scenario and $\mathcal{A}_i^s\neq \emptyset$ for all $t \geq t_1$.
\end{proof}

Next, before deriving our energy-minimizing constraint, we present the unique equilibrium point that minimizes the energy consumption of each CAV $i\in\mathcal{N}(t)$.
For the lead vehicle, i.e., $i = 0$, the drag force is minimized at $v_{0} = 0$ and increasing in $v_{0}$ by Assumption \ref{smp:aerodynamics}.
It is trivial to show that the lead vehicle's energy consumption is minimized at $v_{0}(t) = v_{\min}$.
For a following vehicle, i.e., $i > 0$, the Karush-Kuhn-Tucker (KKT) conditions yield
\begin{align}
    L &= F^2 + \mu^v(v_{\min} - v_i) + \mu^p(\hat{p}_i + \delta), \\
    \frac{\partial L}{\partial v_i} &= 2 F F_v - \mu^v = 0, \\
    \frac{\partial L}{\partial \hat{p}_i} &= 2 F F_{\hat{p}} + \mu^p = 0, \\
    \frac{\partial L}{\partial \mu^v} &= v_{\min} - v_i = 0, \\
    \frac{\partial L}{\partial \mu^p} &= \hat{p}_i + \delta = 0,
\end{align}
where the subscripts $v$, $\hat{p}$ refer to the partial derivative of $F$ with respect to $v_i(t)$ and $\hat{p}_i(t)$, respectively, and $\hat{p}_i(t) + \delta = 0$ is implied by \eqref{eq:rearSafety} when $v_i(t)=v_{\min}$. 
Given Assumption \ref{smp:aerodynamics}, we can determine the signs of the partial derivatives, which implies
\begin{align}
    v_i &= v_{\min}, & \hat{p}_i &= -\delta, \\
    \mu^v &= 2 F \frac{\partial F}{\partial v_i} > 0, & \mu^p &=  -2F \frac{\partial F}{\partial \hat{p}_i} > 0.
\end{align}
Thus, CAV $i>0$ minimizes its energy consumption by following CAV $i-1$ at speed $v_{\min}$ and distance $\hat{p}_i(t) = \delta$.
Note that, as $F_i$ is strictly increasing in $v_i$ and strictly decreasing in $\hat{p}_i$, thus the platooning formation corresponds to the unique minimum-energy configuration of the $N$ CAVs.

Finally, to minimize the drag force imposed on each vehicle, we implement gradient flow by requiring the time derivative of the cost functional \eqref{eq:objective} to be negative semidefinite for each vehicle $i\in\mathcal{N}(t)$.
For vehicle $i=0$ this implies
\begin{equation}
    \dot{J}_i = \frac{\partial F\big(v_i(t), \hat{p}_i(t)\big)}{\partial v_i(t)} a_i(t) \leq 0,
\end{equation}
which, by Assumption \ref{smp:aerodynamics}, implies that
\begin{equation} \label{eq:flowConstraint0}
    a_i(t) \leq 0 \quad \text{ for } i = 0.
\end{equation}
For vehicle $i > 0$, Assumption \ref{smp:aerodynamics} implies
\begin{align}
    \dot{J}_i =
    \begin{bmatrix}
    (F\, F_{v}) \\
    (F\, F_{\hat{p})}
    \end{bmatrix}
    \cdot
    \begin{bmatrix}
    a_i(t) \\
    \hat{v}_i(t)
    \end{bmatrix} \leq 0. \label{eq:jDotMatrix}
\end{align}
Expanding \eqref{eq:jDotMatrix} yields
\begin{align}
    F\,F_{v}\,a_i(t) + F\,F_{\hat{p}}\,\hat{v}_i(t) \leq 0,
\end{align}
which can be solved for $a_i(t)$ using the signs of the partial derivatives imposed by Assumption \ref{smp:aerodynamics},
\begin{align} \label{eq:flowConstraint}
    a_i(t) \leq \frac{|F_{\hat{p}}|}{F_v} \hat{v}_i(t).
\end{align}

Thus, in order for CAV $i$ to form a platoon with $i-1$, we must have $\hat{v}_i(t) > 0$.
Intuitively this makes sense, if $\hat{v}_i(t) < 0$ then $\hat{p}_i(t)$ is decreasing (increasing the drag force) and $i$ must decelerate to achieve an equivalent decrease in the drag force.
Likewise, $\hat{v}_i(t) > 0$ implies that $\hat{p}_i(t)$ is increasing (decreasing the drag force) and $i$ may accelerate without increasing the overall drag force.

Note that, consistent with multi-agent control barrier functions \cite{Wang2017SafetySystems}, it is possible that imposing \eqref{eq:flowConstraint} and the set of safe control inputs (Definition \ref{def:safeset}) on each vehicle admits no feasible solutions.
In particular, this occurs when
\begin{align}
    v_i(t) = v_{\min} \text{ and }  & \frac{|F_{\hat{p}}|}{F_v} \Big( v_{\min} - v_{i-1}(t) \Big) < 0, \label{eq:infeasible1}\\
    a_i(t) &\leq \frac{|F_{\hat{p}}|}{F_v} \hat{v}_i(t) < a_{\min}. \label{eq:infeasible2}
\end{align}
Similar conditions arise when imposing the deadline constraint on CAV $i$.
Taking the time derivative of \eqref{eq:deadline} yields,
\begin{equation}
    -v_i(t) + v_i(t) - a_i(t) \leq 0,
\end{equation}
which implies $i$ must apply $a_i(t) \geq0$ when \eqref{eq:deadline} is active.
It is possible that vehicle $i$ cannot jointly satisfy the deadline, safety, and drag force constraints.
In particular, if either of
\begin{align}
     S_i - p_i(t) - (t_i^f - t)\,v_i(t) = 0 \text{ and }  & \frac{|F_{\hat{p}}|}{F_v} \hat{v}_i(t) < 0, \label{eq:infeasible3}\\
    S_i - p_i(t) - (t_i^f - t)\,v_i(t) = 0 \text{ and } & g_i^s = 0, \hat{v}_i(t) > 0, \label{eq:infeasible4}
\end{align}
is satisfied, then no control action can guarantee drag minimization, safety, and arrival time simultaneously.
Thus, if CAV $i > 0$ satisfies \eqref{eq:infeasible1} - \eqref{eq:infeasible3}, it must fall back and become the lead CAV of its own platoon.
CAV $i$ will re-initialize itself as index $0$ of a new platoon, and all following CAVs $j > i$ will be re-initialized as $j - i$.
This platoon will operate independently as long as any of \eqref{eq:infeasible1} - \eqref{eq:infeasible3} are satisfied for the vehicle physically ahead of this CAV on the road.
The same test may be applied to determine when two platoons ought to merge into a single platoon.
Similarly, if \eqref{eq:infeasible4} is satisfied, then CAV $i$ is unable to achieve its deadline without violating rear-end safety.
This affords at least $2$ possibilities for CAV $i$, 1) move into a passing lane to overtake the preceding vehicle, or 2) relax the deadline constraint until $i$ becomes the lead CAV of a platoon.
Resolving this conflict depends on the geometry of the roadway and application of interest and is beyond the scope of this paper.
Thus, \eqref{eq:infeasible1} - \eqref{eq:infeasible4} determine whether platooning is an appropriate strategy for CAV $i$.

In addition to the above challenges that arise from the task constraint, selecting an energy-minimizing control law that satisfies \eqref{eq:safeControl} and \eqref{eq:flowConstraint} is, in general, insufficient to generate emergent platooning behavior.
This fact is demonstrated in \cite{Shi2021AreFormation}, which shows that only minimizing energy consumption is not a stable configuration for selfish energy-minimizing agents.
As an illustrative example, consider the case where the initial states of the vehicles are randomly selected from the set of feasible states such that each CAV $i\in\mathcal{N}\setminus\{0\}$ is in the wake of vehicle $i-1$.
This implies a transient period for $i$, where $v_{i-1}(t) > v_{\min}$.
To generate a platoon, we would like to have CAV $i$ achieve and maintain $\hat{v}_i(t) > 0$.
We can consider two cases, for the first case let $\hat{v}_i(t_i^0) > 0$, then $i$ can maximize its energy savings by selecting $a_i(t) = a_{\min}$.
However, this may lead to a situation where $\hat{v}_i(t) = 0$ and $\hat{p}_i(t) + \delta <0$, i.e., the vehicles do not form a platoon.
Thus, vehicle $i$ ought to apply a small, feasible deceleration such that $\hat{v}_i(t) > 0 $ is maintained.
In the second case let $\hat{v}_i(t_i^0) \leq 0$, then $i$ ought to decelerate as little as possible, i.e., \eqref{eq:flowConstraint} should be a strict equality.
Then, if CAV $i-1$ applies a large deceleration, it is possible that $\hat{v}_i(t) > 0$ in the future, and $i$ will be able to join the platoon.
The solution of the following optimization problem can accomplish this behavior.

\begin{problem}\label{prb:controlLaw}
For each CAV $i\in\mathcal{N}\setminus\{0\}$, such that \eqref{eq:infeasible1} and \eqref{eq:infeasible2} are not satisfied, generate the control action that solves
\begin{align*}
    &\min_{a_i(t)} \, \frac{1}{2} a_i(t)^2 \\
    \text{subject to:}&\\
    &a_i(t) \in\mathcal{A}_i^s\big(v_i(t), \hat{p}_i(t), \hat{v}_i(t)\big),\\
    &a_i(t) \leq \frac{|F_{\hat{p}}|}{F_v} \hat{v}_i(t), \\
    &\big(S_i - p_i(t)\big) - v_i(t)\,(t_i^f -t) = 0 \implies a_i(t) \geq 0.
\end{align*}
\end{problem}

Note that each vehicle must solve Problem \ref{prb:controlLaw} to determine its control input at each time step.
In this case, the feasible region is compact, and the solution can be derived offline by determining the upper and lower bound on the feasible space of Problem \ref{prb:controlLaw}.
The optimal solution is the feasible value closest to $0$.
Next, we present our main results that characterize sufficient conditions for platoon formation.

\begin{lemma} \label{lma:upperBound}
For any vehicle $i\in\mathcal{N}(t)$ at any time $t\in\mathbb{R}$, the control action that solves Problem \ref{prb:controlLaw} is upper bounded by $0$.
\end{lemma}

\begin{proof}
For vehicle $i=0$, \eqref{eq:flowConstraint0} implies $a_i(t) \leq 0$.

For vehicle $i>0$, let $\alpha_1 = \sup\Big\{\mathcal{A}_i^s\Big\}$, let $\alpha_2 = \frac{|F_{\hat{p}}|}{F_v}\hat{v}_i(t)$, and let $\alpha = \min\{\alpha_1, \alpha_2\}$, i.e., $\alpha$ is the smallest upper bound of Problem \ref{prb:controlLaw}'s feasible space.
For the case when $\alpha \leq 0$, the solution of Problem \ref{prb:controlLaw} is upper bounded by $0$.
For the case when $\alpha > 0$ the lower bound of Problem \ref{prb:controlLaw} is
\begin{equation}
    \beta = 
    \begin{cases}
        a_{\min} & \text{ if } v_i(t) \neq v_{\min}, \\
        0        & \text{ if } v_i(t) = v_{\min},
    \end{cases}
\end{equation}
thus $\beta \leq 0 < \alpha$.
This implies that any control action $a_i(t) > 0$ incurs a higher cost than $a_i(t) = 0$, which is a feasible action in this case.
Thus, the solution of Problem \ref{prb:controlLaw} is always upper bounded by zero.
\end{proof}

\begin{theorem}\label{thm:sufficient}
For two CAVs $i,i-1 \in \mathcal{N}(t)$ the initial condition $v_i(t_i^0) > v_{i-1}(t_i^0)$ guarantees that $i$ and $i-1$ will form a platoon as long as $t_i^f$ and $t_{i-1}^f$ are sufficiently large and the deadline constraint for $i$ does not become active.
\end{theorem}

\begin{proof}
First, consider the case when the rear-end safety constraint does not become active.
Assume $\hat{v}_i(t_1) < 0$ at some $t_1 > t_i^0$.
Continuity of $\hat{v}_i(t)$ implies that there is at least one non-zero interval of time $[t_0, t_1]$ such that $\hat{a}_i(t) < 0$ and $\hat{v}_i(t) \geq 0$ for $t\in[t_0, t_1]$.
Over any such interval, $\hat{v}_i(t) \geq 0$ implies that $a_i(t) = 0$ is a feasible control action.
Furthermore, Lemma \ref{lma:upperBound} implies $a_{i-1}(t) \leq 0$, which implies that $\hat{a}_i(t) \geq 0$.
This contradicts $\hat{a}_i(t) < 0$, therefore no such interval can exist and $\hat{v}_i(t) > 0$ for all $t > t_i^0$ as long as the safety constraint does not become active.

Next, consider the case when only the rear-end safety constraint is active, i.e., \eqref{eq:rearSafety} is strictly equal to zero.
In this case, solving \eqref{eq:rearSafety} for $\hat{v}_i(t)$ yields
\begin{align} \label{eq:vhat2}
    \hat{v}_i(t) = v_i(t) - v_{\min} &- \Big[ \big(v_i(t) - v_{\min}\big)^2 \notag\\
    &+ 2 |a_{\min}|\big(\hat{p}_i(t) + \delta\big) \Big]^{\frac{1}{2}},
\end{align}
where $\hat{p}_i(t) + \delta \leq 0$, and thus \eqref{eq:vhat2} implies $\hat{v}_i(t) \geq 0$ when $\hat{p}_i(t) + \delta < 0$.
Thus, $\hat{v}_i(t)$ is positive and decreasing and only reaches zero when $\hat{p}_i(t) + \delta = 0$, i.e., platoon formation occurs.
\end{proof}

Theorem \ref{thm:sufficient} is a sufficient condition for platooning, and can be recursively applied at any time $t_0$ to guarantee the convergence of any sequence of vehicles satisfying $v_i(t_0) < v_{i+1}(t_0) < \dots < v_{i+k}(t_0)$ for $k\in\mathbb{N}$.
We also note that platooning may occur when $v_i(t_0) > v_{i+1}(t_0)$, in particular if vehicle $i$ decelerates sufficiently fast such that $v_i(t_1) < v_{i+1}(t_1)$ for some $t_1 > t_0$.
In this case, Theorem \ref{thm:recursive} can be applied at $t=t_1$ to guarantee platoon formation.

Finally, the behavior of the front CAV $i=0$ depends on the context of the platooning problem.
The lead CAV may select any trajectory satisfying $a_i(t) \leq 0$ and $v_i(t) \geq v_{\min}$ under our framework.
For example, following $u_i(t) = 0$ could minimize transient energy operation while the drag force slows the vehicle down to the minimum speed.
Alternatively, to facilitate platoon formation, it may be practical to select $a_i(t) = a_{\min}$ to reach the minimum speed as fast as possible.
We apply the latter approach in the next sections to demonstrate emergent platoon formation in a simulated and physical experiment.

\section{Simulation Results} \label{sec:sim}

To validate our proposed control approach, we simulated
a road $1750$ m long with $3$ on and off ramps.
The on-ramps were located at $100$, $600$, and $1100$ m, and the off-ramps were at $500$, $1000$, and $1500$ m.
We simulated the flow of traffic over \blue{$140$} seconds, and we introduced vehicles to the system with a random delay $T \sim \mathcal{U}\big(0.5, 1.5\big)$ seconds.
For each CAV $i$, we selected its initial and exit positions from a uniform distribution over the four possible locations, i.e., the three on-ramps and an initial position of of $p_i(t_i^0) = 0$.
Similarly, $i$ may exit the highway at a distance of $p_i(t_i^f) = 1750$ or at any off-ramp beyond $p_i(t_i^0)$.
After selecting its initial position, we discarded any CAV that could not simultaneously satisfy \eqref{eq:vConst}, \eqref{eq:aConst}, and \eqref{eq:rearSafety} for itself and the vehicle behind it.
This approach resulted in $N = \blue{136}$ vehicles entering the highway over $\blue{140}$ seconds, yielding an average inflow of $3\blue{5}00$ vehicles per hour.

We selected the arrival time for each vehicle after determining its feasible initial state.
For each CAV $i$, we drew the arrival time $t_i^f$ from the uniform distribution,
\begin{equation}
    t_i^f \sim \mathcal{U} \Big( \frac{S_i - p_i(t_i^0)}{ v_i(t_i^0)}, \frac{S_i - p_i(t_i^0)}{ v_{\min}} \Big),
\end{equation}
which guaranteed satisfaction of the deadline constraint \eqref{eq:deadline} at $t_i^0$.
In the case that CAV $i$ later was unable to achieve its deadline, i.e., \eqref{eq:infeasible3} or \eqref{eq:infeasible4} became active, we relaxed the deadline constraint.
In particular, when $i$ satisfied \eqref{eq:infeasible3} or \eqref{eq:infeasible4} we removed the deadline constraint from Problem \ref{prb:controlLaw} for $i$.
If $i$ later became the leader of a platoon, we relaxed the drag minimization constraint \eqref{eq:flowConstraint} and required $i$ to accelerate until the deadline constraint \eqref{eq:deadline} was satisfied.
This achieved a balance between energy-minimization and deadline satisfaction while guaranteeing safety, and it circumvented the additional challenges of overtaking in a multi-lane highway environment.

The \blue{vehicle} trajectories are presented in Fig\blue{s}. \ref{fig:PvsT} \blue{and \ref{fig:PvsT-trans}}, which show the dynamic formation and break-up of platoons as vehicles enter and exit the system \blue{over two 60 second windows of the simulation}.
Figs. \ref{fig:PvsT1} show a zoomed in region of Fig. \ref{fig:PvsT} where vehicles entering at the $100$ m on-ramp form a platoon at approximately $\blue{175}$ m.
For further supplemental diagrams, videos of an experimental demonstration, and an in-depth discussion of the simulation see:  \url{https://sites.google.com/view/ud-ids-lab/cdp}.

\begin{figure}[ht]
    \centering
    \includegraphics[width=.9\linewidth]{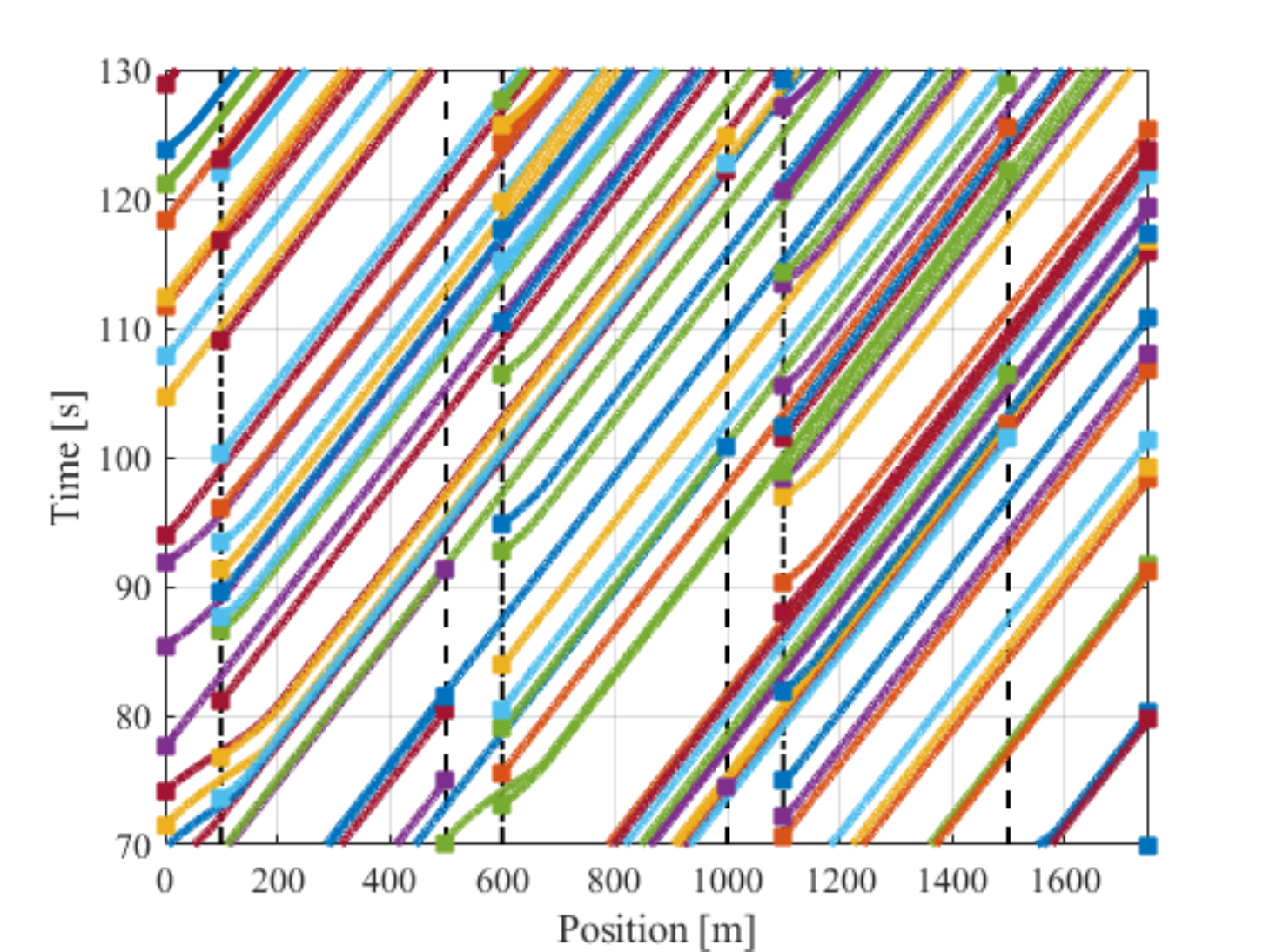}
    \caption{Position vs time plot for the \blue{$N=136$} CAVs \blue{over a $60$ second window of steady operation}. Squares correspond to vehicles entering and exiting the roadway; dash-dot lines correspond to on-ramps and dotted lines correspond to off-ramps.}
    \label{fig:PvsT}
\end{figure}

\begin{figure}[ht]
    \centering
    \includegraphics[width=.9\linewidth]{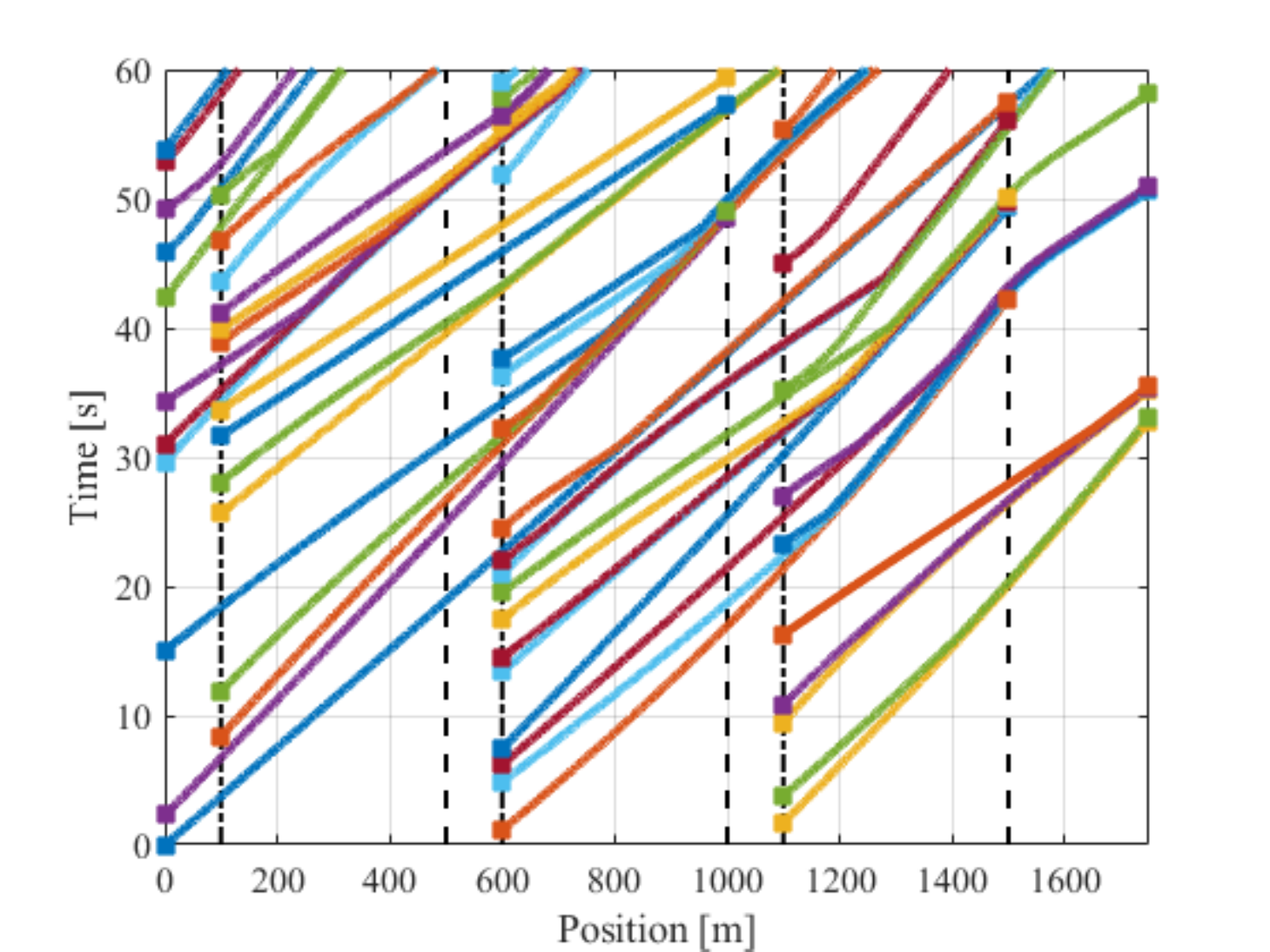}
    \caption{Position vs time plot for the \blue{$N=136$} CAVs \blue{over the initial $60$ second transient}. Squares correspond to vehicles entering and exiting the roadway; dash-dot lines correspond to on-ramps and dotted lines correspond to off-ramps.}
    \label{fig:PvsT-trans}
\end{figure}

\begin{figure}[ht]
    \centering
    \includegraphics[width=.85\linewidth]{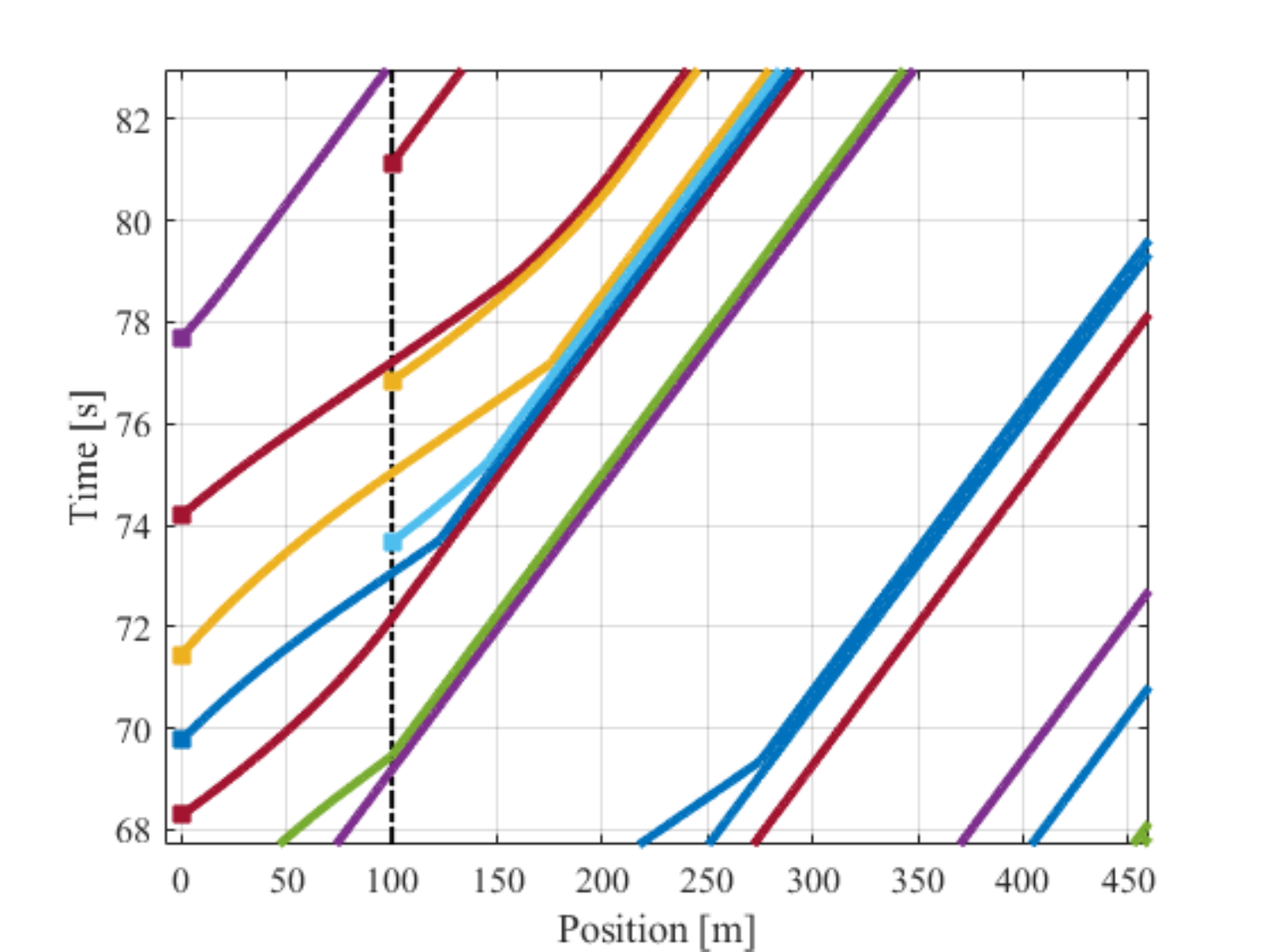}
    \caption{A close up where $4$ vehicles form a platoon near the on-ramp at $100$ m.}
    \label{fig:PvsT1}
\end{figure}

\section{Conclusion} \label{sec:conclusion}

In this article, we derived rigorously a decentralized control law to generate emergent platooning behavior in an open transportation network.
We derived the conditions that determine when platooning is an appropriate strategy, and proved that our proposed control law satisfies recursive feasibility.
We presented a sufficient condition that guarantees platooning, and demonstrated the performance of our descriptive control law in a simulation of an open transportation network, i.e., where vehicles can freely enter and exit.
Future work includes extending our analysis to $\mathbb{R}^2$, with applications to bicycle-riding agents, off-road vehicles, \blue{ and multi-lane overtaking}.
\blue{Finally, extensive simulations on larger-scale systems would be of value, in addition to experiments that capture the magnitude of noise and disturbances}. 


\bibliographystyle{IEEEtran}
\bibliography{mendeley, IDS_Pubs, other-refs}

\end{document}